\documentclass[conference]{IEEEtran}
\IEEEoverridecommandlockouts
\usepackage[sort&compress,numbers]{natbib}
\usepackage{csquotes}
\usepackage{amsmath,amssymb,amsfonts}
\usepackage{algorithmic}
\usepackage{graphicx}
\usepackage{textcomp}
\usepackage{tikz} 
\usepackage{xcolor}
\definecolor{promptboxcolor}{RGB}{193, 255, 193} 
\usetikzlibrary{positioning} 
\usepackage{tabularx}
\usepackage{booktabs}
\usepackage{amsmath, amsthm, amsfonts}

\usepackage{mathtools}

\theoremstyle{definition}
\newtheorem{definition}{Definition}
\usepackage{tabularx, booktabs}

\usepackage{amsmath, amsfonts, amssymb, amsthm}
\usepackage{geometry}
\newtheorem{theorem}{Theorem}

\def\BibTeX{{\rm B\kern-.05em{\sc i\kern-.025em b}\kern-.08em
    T\kern-.1667em\lower.7ex\hbox{E}\kern-.125emX}}

\begin{document}

\title{The Illusion of Equivalency: Statistical Characterization of Quantization Effects in LLMs}

\author{
\IEEEauthorblockN{
Baha Rababah$^{1,2}$,
Shahzeb Qamar$^3$,
Lorenz Sparrenberg$^3$,
Rafet Sifa$^3$\\
Murat Kantarcioglu$^4$,
Cuneyt Gurcan Akcora$^5$,
Carson K. Leung$^1$
}
\IEEEauthorblockA{
\textit{$^1$ Department of Computer Science, University of Manitoba},
Winnipeg, MB, Canada\\
\textit{$^2$ Applied Computer Education Department, Red River College Polytechnic},
Winnipeg, MB, Canada\\
\textit{$^3$ Bonn-Aachen International Center for IT (b-it), University of Bonn},
Bonn, Germany\\
\textit{$^4$ Department of Computer Science, Virginia Tech University},
Blacksburg, Virginia, USA\\
\textit{$^5$ AI Initiative, University of Central Florida},
Orlando, FL, USA\\
rababahb@myumanitoba.ca, 
shahzeb.qamar@iais.fraunhofer.de, 
lsparren@uni-bonn.de\\
Rafet.Sifa@bit.uni-bonn.de, 
muratk@vt.edu, 
cuneyt.akcora@ucf.edu, 
carson.leung@umanitoba.ca
}
}

\maketitle

\begin{abstract}
Post-Training Quantization has become widely used to compress large language models to make them deployable on resource-constrained devices. However, the evaluation of quantization methods mainly uses accuracy and perplexity, which cannot capture the behavioral changes in the quantized variants. In this work, we propose Correctness Agreement, a decision-level metric that can measure the intersection of correct predictions between the base model and its quantized variant. We use this metric across multiple models and quantization bit levels (8-bit to 2-bit), and we find that the base and quantized variants usually have a shift in behavior even when accuracy and perplexity are preserved. In order to explain this effect, we study the effect of quantization on the structure of the attention weights using statistical and distributional measures. The results reveal a breakpoint at low bit widths and show that query and key projections are more sensitive to quantization than the value and output projections. These results prove the illusion of equivalency between the base and quantized models and inspire behavioral evaluation beyond perplexity and accuracy for quantization methods.
\end{abstract}

\begin{IEEEkeywords}
Big data, Large Language Models (LLMs)
\end{IEEEkeywords}

\section{Introduction}
\label{sec:Introduction}

Recently, advancements in large language models (LLMs) have enhanced text generation and reasoning capabilities~\cite{ziyu2023through}, but their large size makes their deployment on resource-constrained devices challenging~\cite{Husom2025Sustainable}. Quantization converts a pretrained model parameters from floating points to lower-precision representations after training. It has emerged as an effective AI approach to reduce memory footprint, accelerate inference, and reduce energy consumption~\cite{jin2024Comprehensive}. However, quantization can degrade performance and change model behavior, especially under low-bit precision settings such as 2-bits~\cite{dutta2024Accuracy}.

The main objective of quantization is to keep the functional behavior of the original model while reducing the model's size. In practice, behavioral preservation can not be measure by perplexity or task-level accuracy to include factual knowledge, reasoning robustness, consistent stylistic and safety-related behaviors to measure the behavioral stability and reliability. These properties are especially vulnerable to quantization, as small numerical changes in model parameters may lead to unexpected behavioral changes~\cite{dutta2024Accuracy}. Nowadays, existing quantization evaluations frameworks~\cite{Kurtic2025give,Li2025quantization,jin2024Comprehensive} use surface-level metrics such as accuracy, loss, and perplexity. These metrics do not fully capture whether quantization preserves behavioral consistency between the base model and its quantized version. In particular, models with similar accuracy or perplexity may still differ substantially in their predictions on individual examples. For example, A model could achieve $75\%$ accuracy at base model and $75\%$ accuracy at quantized model does not necessarily mean that the same predictions were preserved. For example, $15\%$ of the questions that were correctly answered at based could become incorrect at quantized, while another $15\%$ of previously incorrect questions could become correct. Thus, the overall accuracy remains unchanged despite substantial changes in individual predictions.

To address this limitation, we introduce \textit{correctness agreement}, a behavioral evaluation metric that measures the consistency of correct predictions between the original model and its quantized variants. It tracks whether the exact same individual questions remain correct across quantization levels. Unlike aggregate performance metrics, our metric directly assesses whether quantization preserves decision-level behavior. We also provide a deep analysis of weight changes across a wide range of quantization levels (8-bit to 2-bit). This enables a deep understanding of quantization as a continuous structural transformation and reveals how internal weight statistics and divergences changes with lower precision. Our analysis disclose non-linear breakpoints where small reductions in bit-width lead to irregular structural distortions, even when accuracy and perplexity appear stable. Furthermore, examining multiple quantization bit-levels further allows us to isolate layer-wise and architecture-dependent sensitivities, identifying which components of LLMs are inherently robust and which are sensitive to aggressive low-bit settings. 

Our {\em key contributions} are as follows:

\begin{itemize}
\item We introduce post-training quantization as an operator family and study its impact on attention projection weights and decision behavior.
\item We propose a statistical and divergence-based sensitivity analysis that quantifies layer-wise structural drift using moments and distributional distances under quantization.
\item We introduce \textit{correctness agreement}, a decision-level metric that measures overlap in correct predictions between a base model and its quantized variants and complements accuracy and perplexity.
\item To our knowledge, this work presents the first controlled sweep from 8-bit to 2-bit across legacy and K quantization that tracks block-level attention projection drift and links it empirically to decision behavior changes under post-training quantization.
\end{itemize}

\section{Background and Related Work}
\label{sec:Related Work}

Quantization reduces the numerical precision of neural network weights from floating point to lower bit representations in order to decrease model size, memory bandwidth usage, and inference cost on resource-constrained hardware~\cite{xiao2023smoothquant}. Existing approaches fall into two categories: Quantization-Aware Training (QAT) and Post-Training Quantization (PTQ)~\cite{Jacob2018Quantization}. QAT incorporates quantization during training and requires access to the training pipeline, making it expensive and impractical. In contrast, PTQ applies quantization directly to pretrained models without retraining, avoiding the cost of large scale optimization. For this reason, PTQ is a better choice in practice and has been widely adopted for compressing LLMs~\cite{dettmers2023spqr,kim2023squeezellm,xiao2023smoothquant,frantar2022gptq}. In this work, we use \texttt{llama.cpp}~\cite{gerganov2024llamacpp}, an open-source framework for PTQ. Its C/C++ implementation provides broad hardware support and can be used in research and production. In this work, we focus on two quantization families implemented in \texttt{llama.cpp}: legacy quantization and K-quantization~\cite{lee2025Quantization}, which together cover the full precision range from 8-bit to 2-bit~\cite{gerganov2024llamacpp}. Legacy quantization methods (e.g. \texttt{Q8\_0}, \texttt{Q5\_0}, \texttt{Q4\_0}) use block-wise uniform linear quantization with per-block scaling over fixed blocks of $k=256$ weights. K-quantization methods (e.g. \texttt{Q6\_K}, \texttt{Q5\_K}, \texttt{Q4\_K}, \texttt{Q3\_K}, \texttt{Q2\_K}) use a statistically informed block-based scheme that normalizes sub-blocks using estimates of the local mean and standard deviation. 

Table~\ref{tab:paper_comparison} summarizes evaluation dimensions across recent studies. Most prior work evaluate the quantization methods using accuracy and perplexity, which do not test whether the quantized model preserves the base model’s correct decisions on the same examples or connect internal weight distortions to decision behavior. Several studies focus on evaluating quantized LLMs that can be categorized into: (1) knowledge and capacity assessment, measuring correctness and reasoning ability \cite{Hendrycks2021Measuring,shi2024CORECODE}; and (2) alignment evaluation, assessing adherence to human preferences and safety \cite{Gehman2020RealToxicityPrompts,yin2023do}. Jin et al. \cite{jin2024Comprehensive} proposed a structured evaluation framework spanning knowledge, alignment, and efficiency, showing that 4-bit quantization can maintain model performance. Another work by Kurtic et al. \cite{Kurtic2025give} analyzed FP8, INT8, and INT4 quantization across the LLaMA-3.1 family, illustrating that FP8 is effectively lossless, INT8 results in minimal accuracy reduction, and INT4 weight only quantization performs better than expected. However, these studies still focus only on task level performance, without systematically linking weight space changes to behavioral consistency or layer-wise vulnerabilities.

\begin{table}
    \centering
     
    \setlength{\tabcolsep}{1pt}
\renewcommand{\arraystretch}{0.85}
 \caption{Comparison of evaluation in \textbf{stat}istical analyses, \textbf{div}ergence metrics, evaluation across multiple quantization \textbf{bits}, \textbf{layer}-wise sensitivity, \textbf{behav}ioral evaluation, and conventional \textbf{perf}ormance metrics.}
    \begin{tabular}{lcccccc}
        \toprule
        Article & Stats & Div. & Bits & Layer & Behav. & Perf. \\
        \midrule
         \cite{jin2024Comprehensive} & $\checkmark$ & \textcolor{red}{$\times$}     & $\checkmark$ & \textcolor{red}{$\times$}     & \textcolor{red}{$\times$}     & $\checkmark$ \\
        \cite{dutta2024Accuracy}    & \textcolor{red}{$\times$}     & $\checkmark$ & \textcolor{red}{$\times$}     & \textcolor{red}{$\times$}     & $\checkmark$ & $\checkmark$ \\
        \cite{Gong2024llmc}         & $\checkmark$ & \textcolor{red}{$\times$}     & \textcolor{red}{$\times$}     & $\checkmark$ & \textcolor{red}{$\times$}     & $\checkmark$ \\
        \cite{Kurtic2025give}        & \textcolor{red}{$\times$}     & \textcolor{red}{$\times$}     & \textcolor{red}{$\times$}     & \textcolor{red}{$\times$}     & \textcolor{red}{$\times$}     & $\checkmark$ \\
        \cite{Husom2025Sustainable} & \textcolor{red}{$\times$}     & \textcolor{red}{$\times$}     & \textcolor{red}{$\times$}     & \textcolor{red}{$\times$}     & \textcolor{red}{$\times$}     & $\checkmark$ \\
        \cite{Gungor2025AQUA}       & \textcolor{red}{$\times$}     & \textcolor{red}{$\times$}     & \textcolor{red}{$\times$}     & \textcolor{red}{$\times$}     & \textcolor{red}{$\times$}     & $\checkmark$ \\
        \cite{zhan2025quantized}    & \textcolor{red}{$\times$}     & \textcolor{red}{$\times$}     & \textcolor{red}{$\times$}     & \textcolor{red}{$\times$}     & \textcolor{red}{$\times$}     & $\checkmark$ \\
        \cite{Li2025quantization}   & \textcolor{red}{$\times$}     & \textcolor{red}{$\times$}     & \textcolor{red}{$\times$}     & \textcolor{red}{$\times$}     & \textcolor{red}{$\times$}     & $\checkmark$ \\
         \cite{zhang2026benchmarking}   & \textcolor{red}{$\times$}     & \textcolor{red}{$\times$}     & \textcolor{red}{$\times$}     & \textcolor{red}{$\times$}     & \textcolor{red}{$\times$}     & $\checkmark$ \\
          \cite{zhao2025benchmarking}   & \textcolor{red}{$\times$}     & \textcolor{red}{$\times$}     & \textcolor{red}{$\times$}     & \textcolor{red}{$\times$}     & \textcolor{red}{$\times$}     & $\checkmark$ \\
        \textbf{Our work}           & $\checkmark$ & $\checkmark$ & $\checkmark$ & $\checkmark$ & $\checkmark$ & $\checkmark$ \\
        \bottomrule
    \end{tabular}%
    \vspace*{-2mm}   
    \label{tab:paper_comparison}
\end{table}

\section{Problem Setup \& Methodology}
\label{sec:Problem_Setup}

In this section, we formalize post-training quantization (PTQ) as an operator on transformer parameters and define two metric spaces: (1) structural drift in weight space and (2) decision-level consistency in output space.

\subsection{Notation and Model Formalism}
\label{subsec:Notation_and_Model_Formalism}
Let $\mathcal{V}$ be a vocabulary and $\mathcal{X} \subseteq \mathcal{V}^\star$ be the set of token sequences. We consider a decoder-only transformer with $N$ blocks and hidden dimension $d$, parameterized by $\theta \in \mathbb{R}^P$, where $P$ is the total parameter count. For an input prompt $x \in \mathcal{X}$ and target sequence $y = (y_1, \dots, y_T) \in \mathcal{V}^\star$, the model induces next-token conditional probability distributions:
\begin{equation}
p_\theta(y \mid x) = \prod_{t=1}^T p_\theta(y_t \mid x, y_{<t}).
\end{equation}

For each transformer block index $i \in [N] := \{1, \dots, N\}$, let $W_i^Q, W_i^K, W_i^V, W_i^O \in \mathbb{R}^{d \times d}$ denote the self-attention projection matrices. We collect all attention parameters into the set:
\begin{equation}
\mathcal{W}_\theta := \left\{ W_i^L : i \in [N], \, L \in \{Q, K, V, O\} \right\}.
\end{equation}

\subsection{Post-Training Quantization (PTQ) Operators}
\label{subsec:PTQ_Operators}

We define a finite set of quantization configurations $\mathcal{C}$, where each $c \in \mathcal{C}$ specifies a bit-width $b(c) \in \{2, 3, 4, 5, 6, 8\}$ and a quantization scheme (e.g., legacy or K-quantization). Let $\mathcal{Q}_c$ and $\mathcal{D}_c$ represent the quantization and dequantization mapping functions, respectively. We model post-training quantization as an operator $T_c: \mathbb{R}^P \to \mathbb{R}^P$ acting on parameter space:
\begin{equation}
T_c(\theta) := \mathcal{D}_c(\mathcal{Q}_c(\theta)) \in \mathbb{R}^P.
\end{equation}

The resulting quantized-dequantized model parameterization is denoted as $T_c(\theta)$, and its corresponding parameter matrix for attention projection layer $L \in \{Q, K, V, O\}$ at block $i \in [N]$ is defined as $\widehat{W}_i^{L,(c)} := W_i^L(T_c(\theta))$. 

For a given block $i$ and projection layer $L$, the entry-wise weight error matrix is defined as:
\begin{equation}
\Delta_{W, i}^{L,(c)} := \widehat{W}_i^{L,(c)} - W_i^L.
\end{equation}
When evaluating across all parameters in the concatenated parameter vector $\theta$, the overall parameter perturbation vector is denoted as $\Delta_\theta^{(c)} := T_c(\theta) - \theta$.

\subsection{Weight-Space Divergence and Statistical Drift}
\label{subsec:Weight_Space_Divergence_and_Statistical_Drift}

To measure how quantization alters internal parameter distributions, we define two classes of weight-space functionals evaluated across block projections $L \in \{Q, K, V, O\}$.

\paragraph{Statistical Moments.}
Let $s(\cdot)$ be a statistical functional (mean, standard deviation, skewness, or kurtosis) evaluated over the multiset of entries of a weight tensor. For block $i$, the block-level statistic is:
\begin{equation}
\overline{s}_i(\theta) := \frac{1}{4} \sum_{L \in \{Q, K, V, O\}} s\left(W_i^L(\theta)\right).
\end{equation}
Aggregating across all $N$ blocks yields the model-level average:
\begin{equation}
\label{eq:stat_mean_score}
\overline{s}(\theta) := \frac{1}{N} \sum_{i=1}^N \overline{s}_i(\theta), \quad \text{and} \quad \overline{s}^{(c)} := \frac{1}{N} \sum_{i=1}^N \overline{s}_i\left(T_c(\theta)\right).
\end{equation}
The overall statistical drift under configuration $c$ is given by $\Delta \overline{s}^{(c)} := \overline{s}^{(c)} - \overline{s}(\theta)$. To isolate specific layer vulnerabilities, we also compute individual projection averages:
\begin{equation}
\label{eq:attentopn_mean_stat}
\overline{s}_L := \frac{1}{N} \sum_{i=1}^{N} s\left(W_i^L\right).
\end{equation}

\paragraph{Distributional Divergence.}
Let $d(\cdot, \cdot)$ be a divergence or distance metric defined over flattened weight vectors $\mathrm{vec}(W)$ or empirical histograms (specifically Kolmogorov--Smirnov statistic, Euclidean distance, Cosine similarity, and Kullback--Leibler divergence). The block-level and model-level divergences are defined respectively as:
\begin{align}
\label{eq:distributional_divergence}
\overline{d}_i^{(c)} := \frac{1}{4} \sum_{L \in \{Q,K,V,O\}} d\left(\mathrm{vec}(W_i^L), \mathrm{vec}(\widehat{W}_i^{L,(c)})\right), \\
\quad \overline{d}^{(c)} := \frac{1}{N} \sum_{i=1}^N \overline{d}_i^{(c)}.
\end{align}

\subsection{Behavioral Consistency Metrics}
\label{subsec:Behavioral_Consistency}

To evaluate model behavior on downstream task datasets $\mathcal{D} = \{(x^{(m)}, y^{(m)})\}_{m=1}^M$, let $g_\theta(x)$ denote the model's predicted output under a deterministic scoring functional.

For each instance $m \in [M]$, we define binary correctness indicators for the base model $M_0$ (parameterized by $\theta$) and quantized model $M_q$ (parameterized by $T_c(\theta)$):
\begin{align}
C_0(x^{(m)}) &= \mathbb{I}\left(g_\theta(x^{(m)}) = y^{(m)}\right), \\
C_q(x^{(m)}) &= \mathbb{I}\left(g_{T_c(\theta)}(x^{(m)}) = y^{(m)}\right).
\end{align}

Standard task-level accuracies are given by the expectations $\text{Acc}(M_0) = \mathbb{E}[C_0(x)]$ and $\text{Acc}(M_q) = \mathbb{E}[C_q(x)]$. However, aggregate accuracy completely ignores instance-level behavioral stability that is, whether $M_q$ preserves correct predictions on the exact same individual prompts relative to $M_0$.

\begin{definition}[Correctness Agreement]
Correctness Agreement ($\mathrm{CA}$) measures the probability that both the base model $M_0$ and the quantized model $M_q$ concurrently output correct predictions relative to the target labels across dataset $\mathcal{D}$:
\begin{equation}
\begin{aligned}
\mathrm{CA}(M_0, M_q)
&\triangleq \mathbb{P}_{x \sim \mathcal{D}}
\left(C_0(x) = 1 \land C_q(x) = 1\right) \\
&= \frac{1}{M} \sum_{m=1}^M
\mathbb{I}\Big(
C_0(x^{(m)}) = 1 \\
&\hspace{2.5cm}\land\; C_q(x^{(m)}) = 1
\Big).
\end{aligned}
\end{equation}
\end{definition}

Unlike aggregate task accuracy, $\mathrm{CA}$ explicitly quantifies preserved correct choices between $M_0$ and $M_q$, isolating true internal decision persistence from coincidental accuracy matching.


\section{Theoretical Analysis of Decision Behavior}
\label{sec:Theoretical_Analysis_Decision_Behavior}
In this section, we analyze the decision-level behavioral relationship between a base model $M_0$ and its quantized variant $M_q$. We prove that standard aggregate accuracy is fundamentally incapable of guaranteeing behavioral stability at the instance level and establish theoretical bounds on Correctness Agreement along with its relation to inter-model decision reliability.

\subsection{Problem Definition \& Joint Probability Matrix}
\label{subsec:Problem_Definition_Joint_Probability_Matrix}

Let $x \in \mathcal{X}$ be an input prompt and $y \in \mathcal{Y}$ be the ground-truth target. Recall the binary correctness indicator random variables $C_0(x), C_q(x) \in \{0, 1\}$ for $M_0$ and $M_q$, respectively. The marginal expected accuracies over data distribution $P(\mathcal{X}, \mathcal{Y})$ are given by:
\begin{align}
\text{Acc}_0 = \text{Acc}(M_0) = \mathbb{E}_{x \sim P}[C_0(x)] = p_{1\cdot}, \\
\quad \text{Acc}_q = \text{Acc}(M_q) = \mathbb{E}_{x \sim P}[C_q(x)] = p_{\cdot 1}.
\end{align}

We model the joint decision space of $(C_0, C_q)$ using a $2 \times 2$ contingency probability matrix:

\begin{table}[htbp]
\centering
\caption{Joint distribution matrix of decision correctness between base model $M_0$ and quantized model $M_q$.}
\label{tab:joint_contingency}
\small
\setlength{\tabcolsep}{3pt}
\begin{tabularx}{\columnwidth}{@{} X c c c @{}}
\toprule
 & $C_q = 1$  & $C_q = 0$  & Total \\
\midrule
$C_0 = 1$  & $p_{11}$ & $p_{10}$ & $p_{1\cdot} = \mathrm{Acc}_0$ \\
\addlinespace
$C_0 = 0$  & $p_{01}$ & $p_{00}$ & $p_{0\cdot} = 1 - \mathrm{Acc}_0$ \\
\midrule
Total & $p_{\cdot 1} = \mathrm{Acc}_q$ & $p_{\cdot 0} = 1 - \mathrm{Acc}_q$ & $1.0$ \\
\bottomrule
\end{tabularx}
\end{table}

The joint outcome probabilities in Table~\ref{tab:joint_contingency} correspond to four distinct behavioral regimes:
\begin{itemize}
    \item $p_{11} = \mathbb{P}(C_0 = 1, C_q = 1)$: \textit{Preserved Correctness} ($\mathrm{CA}$).
    \item $p_{00} = \mathbb{P}(C_0 = 0, C_q = 0)$: \textit{Preserved Error} (both models predict incorrectly).
    \item $p_{10} = \mathbb{P}(C_0 = 1, C_q = 0)$: \textit{Catastrophic Regress} (quantization degrades a correct decision).
    \item $p_{01} = \mathbb{P}(C_0 = 0, C_q = 1)$: \textit{Spurious Progress} (quantization turns an error into a correct prediction).
\end{itemize}

By definition, Correctness Agreement corresponds directly to joint correctness:
\begin{equation}
\mathrm{CA}(M_0, M_q) = p_{11}.
\end{equation}

\subsection{Accuracy Invariance Blindness Theorem}
\label{subsec:Accuracy_Invariance_Blindness_Theorem}

We now prove that matching task accuracies between a base model and its quantized counterpart provides no assurance that individual decision outputs remain consistent.

\begin{theorem}[Accuracy Invariance Blindness]
Equal aggregate accuracies between $M_0$ and $M_q$ do not imply decision stability. The Correctness Agreement $\mathrm{CA}$ can vary widely even when $\Delta \text{Acc} = |\text{Acc}_0 - \text{Acc}_q| = 0$.
\end{theorem}

\begin{proof}
Suppose $\text{Acc}(M_0) = \text{Acc}(M_q) = \alpha$, where $\alpha \in (0, 1]$. From the marginal constraints:
\begin{equation}
p_{1\cdot} = p_{11} + p_{10} = \alpha \implies p_{10} = \alpha - p_{11},
\end{equation}
\begin{equation}
p_{\cdot 1} = p_{11} + p_{01} = \alpha \implies p_{01} = \alpha - p_{11}.
\end{equation}
Since $p_{10} \ge 0$ and $p_{01} \ge 0$, we have $p_{11} \le \alpha$. Furthermore, from the non-negativity of $p_{00} = 1 - 2\alpha + p_{11}$, we require $p_{11} \ge \max(0, 2\alpha - 1)$. 

Substituting $\mathrm{CA}(M_0, M_q) = p_{11}$ yields the valid variation range:
\begin{equation}
\mathrm{CA}(M_0, M_q) \in \left[ \max(0, \, 2\alpha - 1), \; \alpha \right].
\end{equation}
\end{proof}

\noindent\textbf{Practical Implications:} Consider a base model and quantized model both exhibiting $\text{Acc} = 60\%$ ($\alpha = 0.6$). According to Theorem 1, $\mathrm{CA}$ can legally take any value between $\max(0, 1.2 - 1) = 0.20$ ($20\%$) and $0.60$ ($60\%$). Surface-level accuracy completely conceals internal instance-level decision instability.

\subsection{Bounds \& Relation to Cohen's Kappa}
\label{subsec:Bounds_Relation_to_Cohen's_Kappa}

Having demonstrated accuracy blindness, we establish the strict mathematical bounds governing $\mathrm{CA}$ as a function of marginal accuracies and connect it to chance-corrected reliability.

\begin{theorem}[Upper and Lower Bounds of Correctness Agreement]
\label{thm:ca_bounds}
Correctness Agreement $\mathrm{CA}(M_0, M_q)$ is strictly bounded by the aggregate accuracies $\mathrm{Acc}_0$ and $\mathrm{Acc}_q$:
\begin{equation}
\begin{aligned}
\max\left(0, \, \mathrm{Acc}_0 + \mathrm{Acc}_q - 1\right)
&\leq \mathrm{CA}(M_0, M_q) \\
&\leq \min\left(\mathrm{Acc}_0, \, \mathrm{Acc}_q\right).
\end{aligned}
\end{equation}
\end{theorem}

\begin{proof}
\textbf{Upper Bound:} \\
Since $p_{11} \le p_{1\cdot} = \mathrm{Acc}_0$ and $p_{11} \le p_{\cdot 1} = \mathrm{Acc}_q$, $p_{11}$ cannot exceed either marginal accuracy:
\begin{equation}
\mathrm{CA}(M_0, M_q) = p_{11} \le \min\left(\mathrm{Acc}_0, \, \mathrm{Acc}_q\right).
\end{equation}

\textbf{Lower Bound:} \\
By the Fr\'echet inequality for joint probabilities:
\begin{equation}
p_{11} \ge \max\left(0, \, p_{1\cdot} + p_{\cdot 1} - 1\right) = \max\left(0, \, \mathrm{Acc}_0 + \mathrm{Acc}_q - 1\right).
\end{equation}
Combining both bounds yields the complete theoretical range for $\mathrm{CA}$.
\end{proof}

\begin{theorem}[Chance-Corrected Reliability via Cohen's Kappa]
The joint prediction outcomes of $M_0$ and $M_q$ can be mapped to chance-corrected decision alignment using Cohen's Kappa statistic ($\kappa$).
\end{theorem}

\begin{proof}
Let total prediction outcome agreement be $P_o = p_{11} + p_{00} = \mathrm{CA} + p_{00}$. Under the assumption of statistical independence between model prediction indicators, expected random outcome agreement $P_e$ is:
\begin{equation}
\begin{split}
P_e &= (p_{1\cdot} \cdot p_{\cdot 1}) + (p_{0\cdot} \cdot p_{\cdot 0}) \\
&= \text{Acc}_0 \cdot \text{Acc}_q + (1 - \text{Acc}_0)(1 - \text{Acc}_q).
\end{split}
\end{equation}
Evaluating Cohen's Kappa metric:
\begin{equation}
\kappa \triangleq \frac{P_o - P_e}{1 - P_e}.
\end{equation}
When $\kappa = 1$, $M_q$ perfectly reproduces the decision output of $M_0$ across all instances. When $\kappa \le 0$, decision agreement between $M_0$ and $M_q$ is no better than independent random guessing.
\end{proof}

\section{Experimental Results}
\label{sec:experiment}
We quantize the following four models: $\texttt{Llama-3.2-3B}$, $\texttt{Vicuna-7B-v1.5}$, $\texttt{Mistral-7B-v0.1}$, and $\texttt{Llama-3.1-8B}$ into legacy quantization methods (\texttt{Q8\_0}, \texttt{Q5\_0}, \texttt{Q4\_0}) and K-quantization methods(\texttt{Q6\_K}, \texttt{Q5\_K}, \texttt{Q4\_K}, \texttt{Q3\_K}, \texttt{Q2\_K}) using llama.cpp \cite{gerganov2024llamacpp}. We extract the layers' weights of the base models and its quantized variants after dequantization. 

\paragraph{Datasets.}
We evaluate quantized models using both language modeling corpora and downstream benchmark tasks. For perplexity evaluation, we use WikiText-2~\cite{merity2017pointer} and C4~\cite{raffel2020exploring}. To assess task-level performance and behavioral consistency, we use zero-shot benchmarks including HellaSwag~\cite{zellers2019hellaswag}, Winogrande~\cite{sakaguchi2019winogrande}, and ARC (AI2 Reasoning Challenge)~\cite{clark2018think}. These datasets are all multiple-choice benchmarks and cover a range of reasoning capabilities, including commonsense reasoning, pronoun resolution, and science question answering. Specifically, HellaSwag~\cite{zellers2019hellaswag} evaluates commonsense reasoning through context completion, WinoGrande~\cite{sakaguchi2019winogrande} evaluates commonsense reasoning and pronoun resolution, and ARC (AI2 Reasoning Challenge)~\cite{clark2018think} evaluates grade-school science question answering. We evaluated the models on NVIDIA Tesla H100 80GB PCIe. Our Python implementation is included in the submission.

\subsection{Statistical Changes in Attention Layers}
\label{subsec:statistical_Attention_Layers}

For each transformer block $i$, we compute a block-level statistic by averaging the weights of attention matrices: $s(W_i^Q)$, $s(W_i^K)$, $s(W_i^V)$, and $s(W_i^O)$, and then aggregate these values across all $N$ blocks as defined in Eq.~\eqref{eq:stat_mean_score}. This yields a global view of how attention-weight distributions evolve under different quantization levels. The aggregated results are shown in Fig.~\ref{fig:combined_aggregated_stats}.

\begin{figure}[t]
    \centering
    \includegraphics[width=\columnwidth]{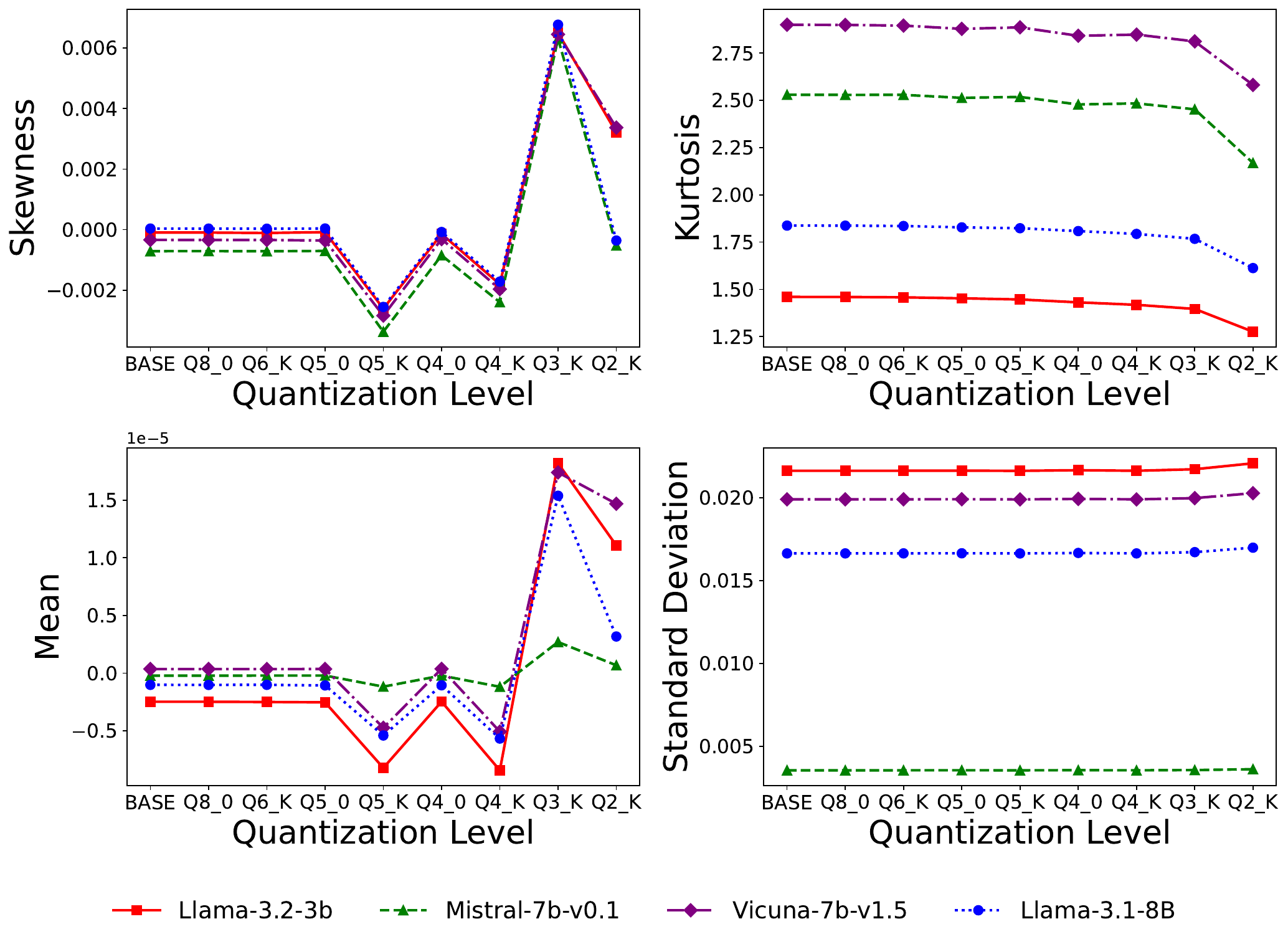}
    \vspace*{-7mm}
    \caption{Aggregated statistical metrics across legacy and K-quantization methods.}
    \label{fig:combined_aggregated_stats}
\end{figure}

Figure~\ref{fig:combined_aggregated_stats} shows that higher-bit quantization schemes (8 to 5 bits) commonly preserve the statistical properties of attention weights, with near-zero mean, low skewness, stable kurtosis, and baseline level standard deviation across models. On the other hand, low-bit quantization (4 to 2 bits), particularly \texttt{Q3\_K} and \texttt{Q2\_K}, introduces large distortions, including increased skewness and standared deviation, reduced kurtosis, and systematic mean shifts. These effects are model dependent, with smaller models such as $\texttt{Llama-3.2-3B}$ exhibiting larger shifts. A notable finding is that \texttt{Q5\_0} closely matches the base model across all statistical metrics, indicating strong preservation of distributional structure at moderate compression.

\paragraph{Individual Statistic Comparison.}
To study the impact of quantization on individual attention projection component, we calculate statistics independently for each attention projection type ($K$, $Q$, $V$, $O$) by averaging $s(W_i^L)$ across all $N$ blocks, as defined in Eq.~\eqref{eq:attentopn_mean_stat}. For legacy quantization schemes (\texttt{Q8\_0}, \texttt{Q5\_0}, \texttt{Q4\_0}), layer-wise statistics remain almost identical to the base model across all metrics, which indicates that legacy quantization has minimal structural impact. On the other hand, K-quantization results in Fig.~\ref{fig:Avg_skewness_K_quant_main_models} for skew and Fig.~\ref{fig:Avg_kurtosis_K_quant_main_models} for kurtosis shows layer dependent sensitivity as precision decreases. Moderate K-quantization levels (\texttt{Q6\_K}–\texttt{Q4\_K}) preserve attention-layer statistics across $Q$, $K$, $V$, and $O$, maintaining stable distribution shape, tail behavior, and variance. However, aggressive low-bit settings (\texttt{Q3\_K} and \texttt{Q2\_K}) lead to large distortions, especially in the $Q$ and $K$ projections, where skewness becomes volatile, kurtosis collapses, and mean and variance deviate sharply from the base model. The $V$ and $O$ projections remain stable under the same conditions, except for the \texttt{Q2\_K}. It is clear from this results that legacy quantization introduces negligible statistical distortion, but K-quantization introduces non-uniform, layer-dependent effects at low bit-widths. In particular, the $Q$ and $K$ projections are the most sensitive to aggressive quantization across all evaluated models.

\begin{figure}[tb]
    \centering
    \includegraphics[width=\columnwidth]{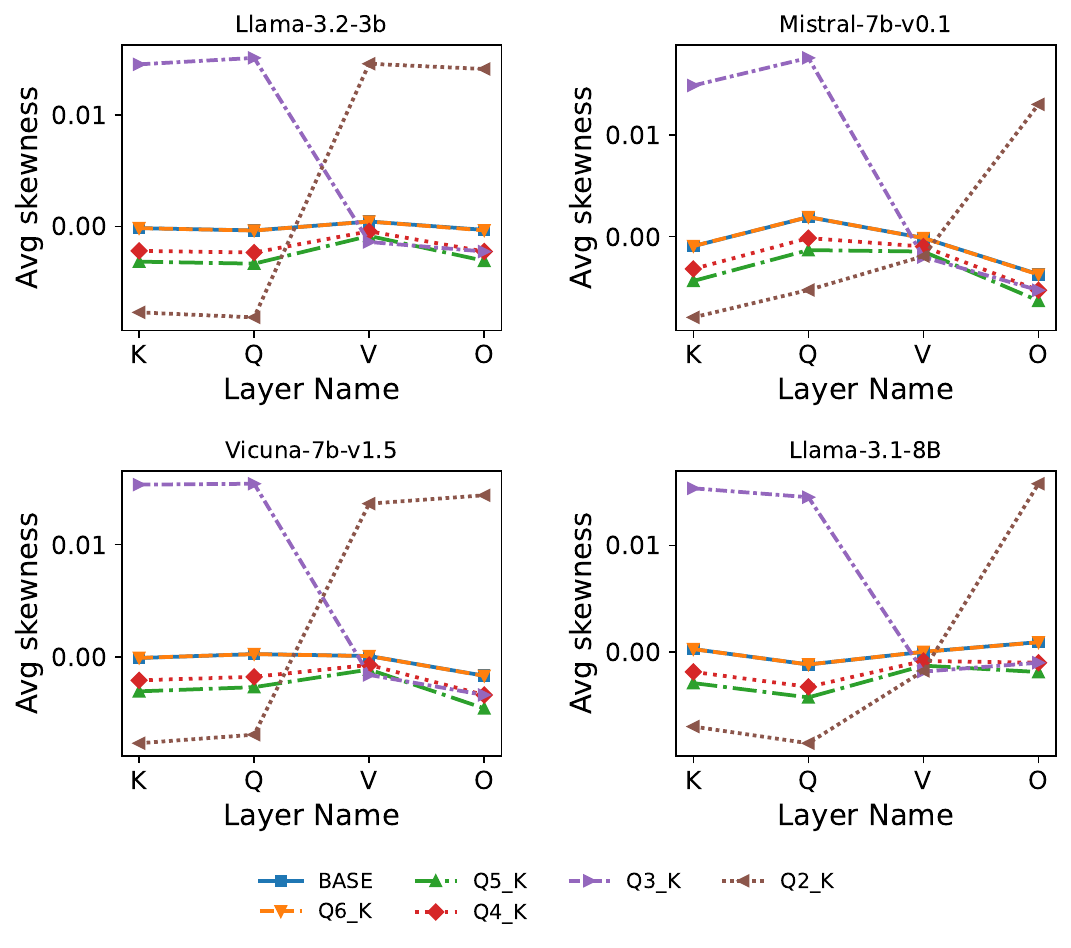}
    \vspace*{-7mm}
    \caption{Average skewness across attention layers under K-quantization.}
    \label{fig:Avg_skewness_K_quant_main_models}
\end{figure}

\begin{figure}[tb]
    \centering
    \includegraphics[width=\columnwidth]{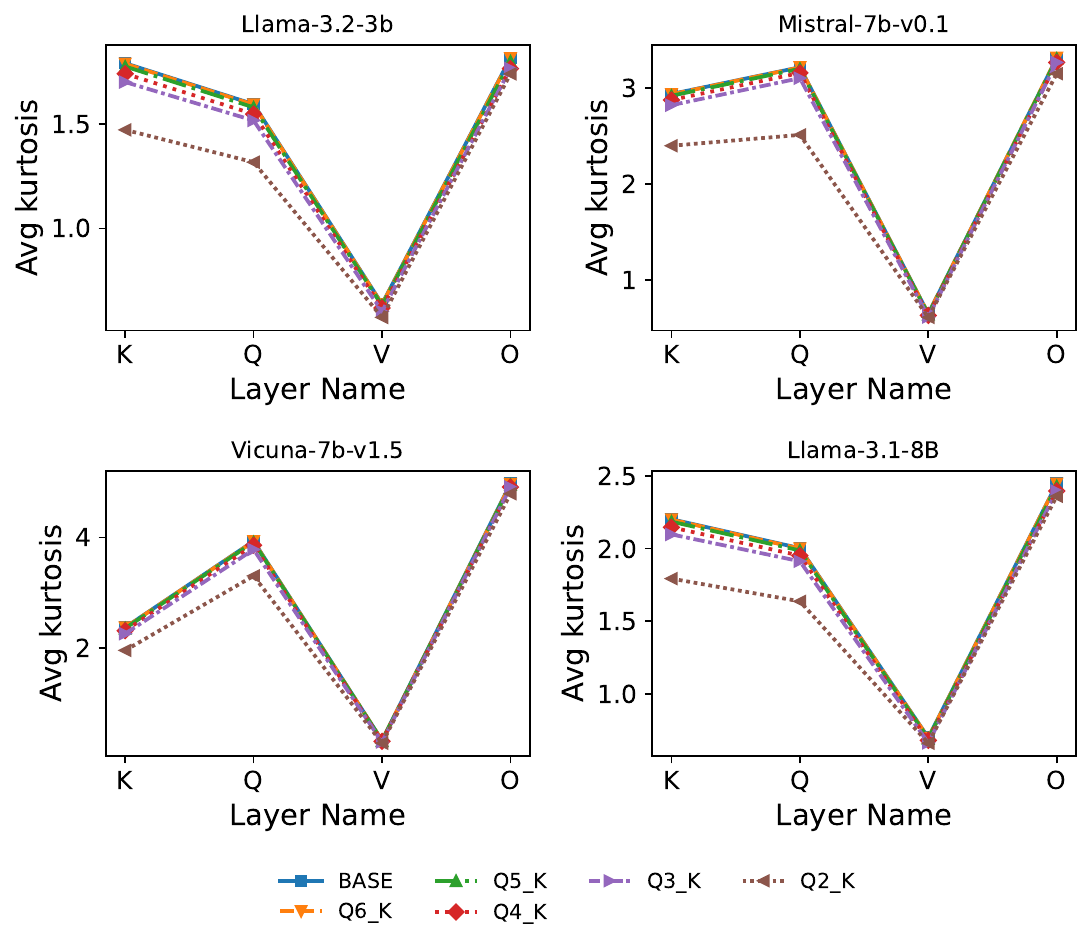}
    \vspace*{-7mm}
    \caption{Average kurtosis across attention layers under K-quantization.}
    \label{fig:Avg_kurtosis_K_quant_main_models}
\end{figure}


\subsection{Distributional Divergence in Attention Layers}

We complement the statistical analysis with distributional divergence measures to study how quantization changes the structure of attention weight distributions. For each transformer block $i$, we compute divergence metrics between the base and quantized weight matrices for each attention projection $L \in \{Q, K, V, O\}$, and aggregate these values across blocks as defined in equation \ref{eq:distributional_divergence}.

\begin{figure}[t] 
    \centering
    \includegraphics[width=\columnwidth]{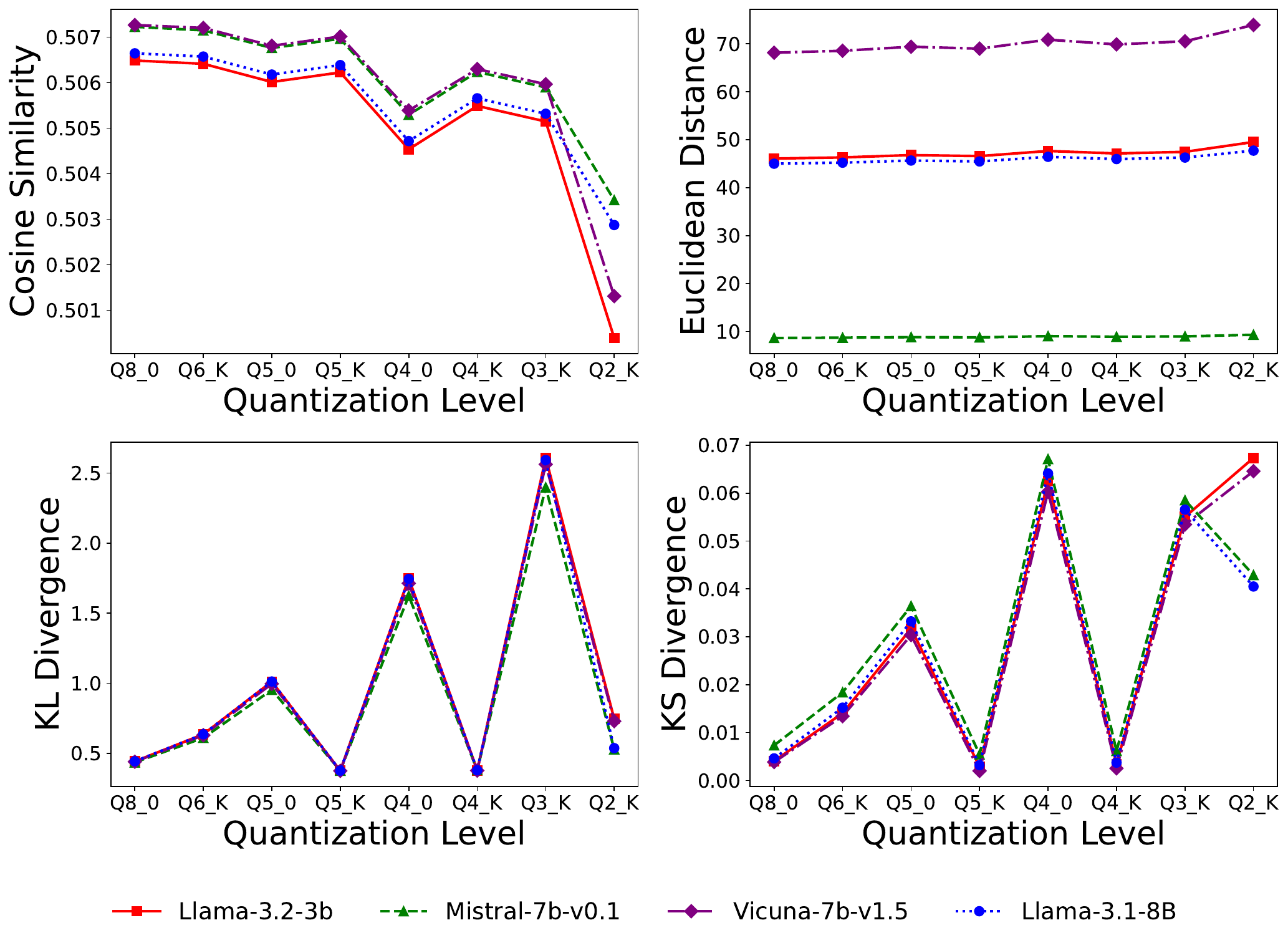}
    \vspace*{-7mm}
    \caption{Comparison of aggregated similarity and divergence metrics across quantization methods.}
    \label{fig:combined_metrics_single_column}
\end{figure}

\paragraph{Aggregated Divergence Across Quantization Levels.}
Figure~\ref{fig:combined_metrics_single_column} presents the aggregated divergence trends across quantization levels for all models. Cosine similarity remains almost stable across high bit quantizations (Q8\_0 through Q5\_K), drops slightly at Q4\_0, recovers at Q4\_K and Q3\_K, and drop sharply at Q2\_K. Euclidean distance varies across quantization settings and is less sensitive to low bit quantization. KL Divergence fluctuates sharply depending on the specific quantization format. Spikes occur at legacy quant or lower-precision quants (Q5\_0, Q4\_0, and a large peak at Q3\_K), whereas K-quants (Q5\_K, Q4\_K) maintain significantly lower KL divergence across all models. KS divergence behave similar to KL divergence, with prominent peaks at Q5\_0, Q4\_0, Q3\_K, and Q2\_K, while Q5\_K and Q4\_K achieve near-zero divergence across models.

\begin{figure}[t] 
    \centering
    \includegraphics[width=\columnwidth]{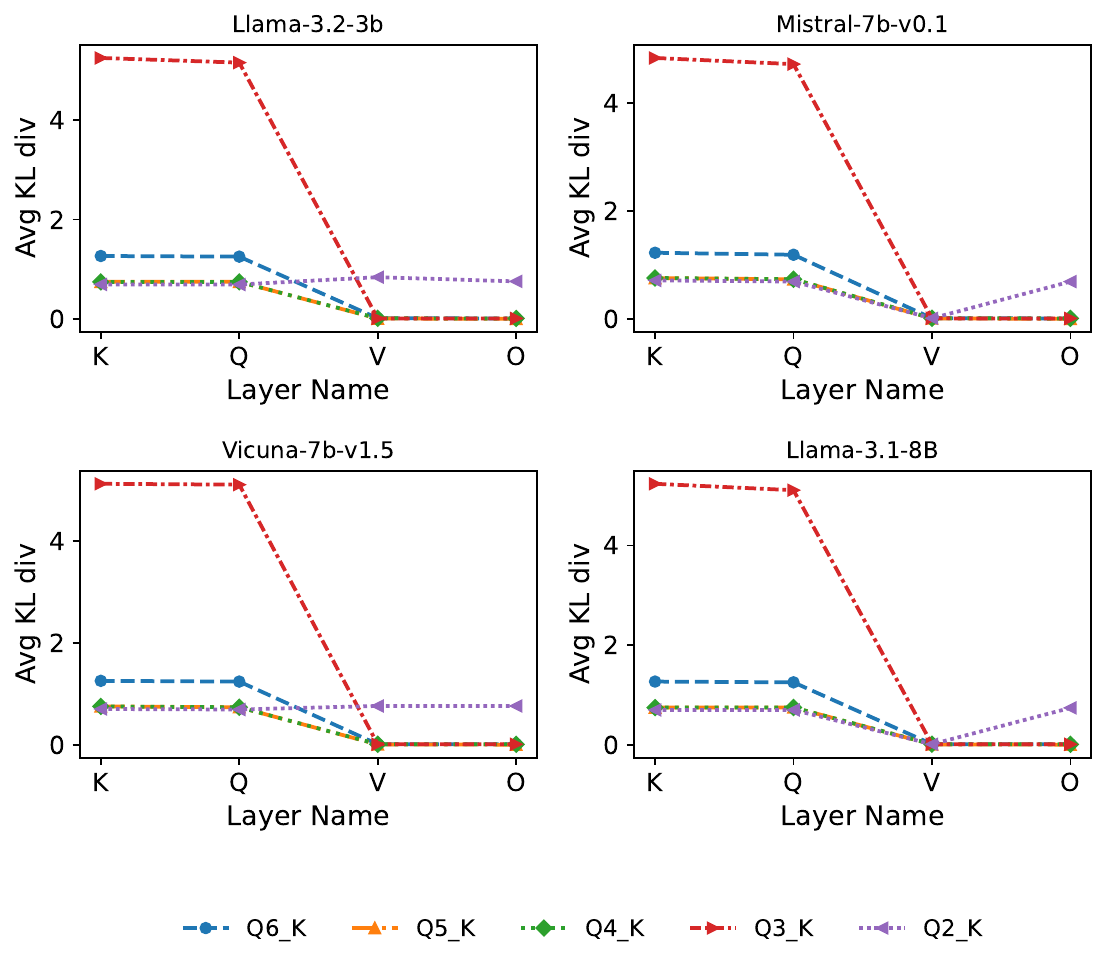}
    \vspace*{-7mm}
    \caption{Average KL  K-Quantization.}
    \label{fig:Avg_KL_div_K_quant_main_models}
\end{figure}

\paragraph{Layer-Specific Divergence.}
To study the impact of quantization on individual attention projection component, we compute divergence metrics individually for each attention projection. Figure~\ref{fig:Avg_KL_div_K_quant_main_models} presents the per layer KL divergence under K-quantization. The $Q$ and $K$ projections exhibit the highest sensitivity to quantization distortion across almost all schemes, where both of them experience a massive spike in KL divergence (exceeding 5.0) when quantized with Q3\_K. Both  $Q$ and $K$ have consistent, lower divergence baseline (~0.75) across Q5\_K, Q4\_K, and Q2\_K. Across Q6\_K, Q5\_K, Q4\_K, and Q3\_K, the KL divergence drops sharply from its peak in K and Q down to 0.0 (or near-zero) at the V projection. Q2\_K does not follow this recovery trend. Instead, its divergence remains around ~0.75 at the V layer. Similar to the V layer, O maintains a KL divergence of 0.0 across Q6\_K, Q5\_K, Q4\_K, and Q3\_K. Unlike all high precision schemes where the output projection shows no divergence change, Q2\_K remains around ~0.75, indicating that 2-bit quantization affects projection differently than 3-bit to 6-bit quantization. 
 
\subsection{Performance Evaluations}
\label{subsec:Evaluations}

\begin{table}[htbp]
\centering
\caption{\textbf{Table 2.} Perplexity ($\downarrow$) on WikiText2.}
\label{tab:perplexity_wikitext2}
\begin{tabular}{lcccc}
\toprule
Quant & Llama-3B & Vicuna & Mistral & Llama-8B \\
\midrule
Base  & 8.16  & 7.31  & 5.54 & 6.45 \\
Q8\_0 & \textcolor{green!60!black}{10.23} & \textcolor{green!60!black}{9.24} & \textcolor{green!60!black}{6.98} & \textcolor{green!60!black}{8.12} \\
Q6\_K & 10.24 & 9.27 & 7.00 & 8.15 \\
Q5\_0 & 10.42 & 9.36  & 7.09 & 8.34 \\
Q5\_K & 10.29 & 9.29  & 7.03 & 8.20 \\
Q4\_0 & 10.87 & 9.62  & 7.27 & 8.73 \\
Q4\_K & 10.50 & 9.35  & 7.06 & 8.35 \\
Q3\_K & 11.44 & 9.58  & 7.22 & 8.98 \\
Q2\_K & \textcolor{red}{16.95} & \textcolor{red}{11.66} & \textcolor{red}{8.40} & \textcolor{red}{12.56} \\
\bottomrule
\end{tabular}
\end{table}
\begin{table}[htbp]
\centering
\caption{\textbf{Table 3.} Perplexity ($\downarrow$) on C4.}
\label{tab:perplexity_c4}
\begin{tabular}{lcccc}
\toprule
Quant & Llama-3B & Vicuna & Mistral & Llama-8B \\
\midrule
Base  & 11.47 & 10.22 & 8.17 & 9.68 \\
Q8\_0 & \textcolor{green!60!black}{12.98} & 11.90 & \textcolor{green!60!black}{9.41} & \textcolor{green!60!black}{10.94} \\
Q6\_K & 13.03 & 11.90 & 9.43 & 10.99 \\
Q5\_0 & 13.25 & 12.06 & 9.45 & 11.25 \\
Q5\_K & 13.10 & \textcolor{green!60!black}{11.89} & 9.45 & 11.05 \\
Q4\_0 & 14.00 & 12.38 & 9.62 & 11.85 \\
Q4\_K & 13.38 & 11.93 & 9.49 & 11.30 \\
Q3\_K & 14.80 & 12.00 & 9.51 & 12.29 \\
Q2\_K & \textcolor{red}{21.50} & \textcolor{red}{13.95} & \textcolor{red}{10.53} & \textcolor{red}{16.96} \\
\bottomrule
\end{tabular}
\end{table}

We evaluate quantized models using both language modeling and downstream benchmark tasks to assess how quantization affects surface-level performance and decision-level behavior.

\paragraph{Perplexity.}
We measure language modeling perplexity on WikiText 2 \cite{merity2017pointer} and C4 \cite{raffel2020exploring}.Perplexity does not follow a consistent trend with respect to bit width, and the effect depends on the model and dataset. Tables~\ref{tab:perplexity_wikitext2} and~\ref{tab:perplexity_c4} show that several mid range quantization preserve perplexity close to the Q8\_0 model for all models. In particular, K-quantization configurations from Q6\_K to Q4\_K avoid the sharp degradation observed under low bit quantization. Q3\_K can produce high perplexity in most cases, but this behavior cannot be generalized because it is model and dataset dependent. The results are proof that perplexity is not a trusted metric for preserved decisions under quantization. In most cases, a clear increase in perplexity is observed for Q2\_K.

\begin{table*}[t]
\centering
\small
\setlength{\tabcolsep}{3.5pt}
\renewcommand{\arraystretch}{1.05}
\caption{Accuracy of ARC, HellaSwag and Winogrande.}
\begin{tabular}{l ccc ccc ccc ccc}
\hline
& \multicolumn{3}{c}{Llama-3B} & \multicolumn{3}{c}{Vicuna-7B} & \multicolumn{3}{c}{Mistral-7B} & \multicolumn{3}{c}{Llama-8B} \\
Quant & Arc & Hell. & Wino. & Arc & Hell. & Wino. & Arc & Hell. & Wino. & Arc & Hell. & Wino. \\
\hline
Base  &0.741&0.415&0.528&0.594&0.374&0.52&0.808&0.457&0.568&0.811&0.432&0.566\\
Q8\_0 &\textcolor{green!60!black}{0.628}&0.432&0.506&0.699&0.488&0.512&0.71&\textcolor{green!60!black}{0.476}&0.507&0.642&\textcolor{green!60!black}{0.468}&\textcolor{green!60!black}{0.514}\\
Q6\_K &0.621&0.432&0.502&0.703&0.485&0.52&0.71&\textcolor{green!60!black}{0.476}&0.506&0.646&\textcolor{green!60!black}{0.468}&0.512\\
Q5\_0 &0.621&0.428&0.507&0.693&0.488&0.51&0.707&0.472&0.508&0.643&0.467&0.501\\
Q5\_K &0.619&0.43&0.509&0.696&\textcolor{green!60!black}{0.49}&\textcolor{green!60!black}{0.521}&0.71&0.475&0.505&\textcolor{green!60!black}{0.647}&0.462&0.513\\
Q4\_0 &0.623&0.427&\textcolor{green!60!black}{0.516}&0.706&0.481&\textcolor{red}{0.508}&0.711&\textcolor{green!60!black}{0.476}&\textcolor{green!60!black}{0.51}&0.646&0.465&0.506\\
Q4\_K &0.62&\textcolor{green!60!black}{0.434}&0.513&0.693&0.485&0.516&\textcolor{green!60!black}{0.712}&0.471&0.5&0.645&\textcolor{green!60!black}{0.468}&0.509\\
Q3\_K &0.604&0.425&0.509&\textcolor{green!60!black}{0.708}&0.483&0.515&0.694&0.465&\textcolor{red}{0.499}&0.646&0.462&0.511\\
Q2\_K &\textcolor{red}{0.542}&\textcolor{red}{0.388}&\textcolor{red}{0.5}&\textcolor{red}{0.643}&\textcolor{red}{0.478}&\textcolor{red}{0.508}&\textcolor{red}{0.682}&\textcolor{red}{0.464}&0.506&\textcolor{red}{0.628}&\textcolor{red}{0.438}&\textcolor{red}{0.5}\\
\hline
\end{tabular}
\label{tab:acc}
\end{table*}

\paragraph{Downstream Accuracy.}
We evaluate zero-shot task performance on ARC~\cite{clark2018think}, HellaSwag~\cite{zellers2019hellaswag} and Winogrande~\cite{sakaguchi2019winogrande}. As shown in Table~\ref{tab:acc}, Across the evaluated models and datasets, quantization from Q8\_0 down to Q4\_K maintains accuracy levels very close to the base models, most of the time accracy are equal or even exceeds the base performance. This trend is particularly appears in Vicuna-7B, where quantized variants outperform the base model across ARC and HellaSwag. This trend also on the HellaSwag across all other model (e.g., Llama-3B improves from a base score of $0.415$ to between $0.434$ and $0.425$ under Q8–Q3 formats). Models' Performance on Winogrande exhibits small degradation within this range, showing small drops such as a $1.2\%$ reduction for Llama-3B (from $0.528$ to $0.516$ at Q4\_0) and a $5.2\%$ drop for Llama-8B at Q8\_0. Transitioning to Q3\_K shows a slight degradation compared to higher bit rates, but it shows solid overall utility and sometimes produces top or lowest performance, such as achieving the top ARC score of $0.708$ for Vicuna-7B. In contrast, Q2\_K represents a clear degradation threshold notably on the ARC benchmark, where performance degrades by $20\%$ for Llama-3B (dropping from $0.741$ to $0.542$) and $12.6\%$ for Mistral-7B (falling from $0.808$ to $0.682$).

All models and their variants are evaluated on the multiple-choice datasets HellaSwag, Winogrande, and ARC using deterministic log-likelihood scoring, without token sampling. In other words, for each question, the model scores the possible answer choices and selects the one with the highest score. The same input produces the same prediction every time because no random sampling is involved during inference, so the repeated runs give same results ($\sigma = 0.000$).

\begin{table*}[t]
\centering
\small
\setlength{\tabcolsep}{3.5pt}
\renewcommand{\arraystretch}{1.05}
\caption{Correctness Agreement of ARC, HellaSwag and Winogrande.}
\begin{tabular}{l ccc ccc ccc ccc}
\hline
& \multicolumn{3}{c}{Llama-3B} & \multicolumn{3}{c}{Vicuna-7B} & \multicolumn{3}{c}{Mistral-7B} & \multicolumn{3}{c}{Llama-8B} \\
Quant & Arc & Hell. & Wino. & Arc & Hell. & Wino. & Arc & Hell. & Wino. & Arc & Hell. & Wino. \\
\hline
Q8\_0 &\textcolor{green!60!black}{0.573}&\textcolor{green!60!black}{0.381}&0.274&0.510&0.342&0.288&0.678&0.435&0.320&0.620&\textcolor{green!60!black}{0.404}&\textcolor{green!60!black}{0.323}\\
Q6\_K &0.569&\textcolor{green!60!black}{0.381}&0.272&\textcolor{green!60!black}{0.511}&0.342&0.294&\textcolor{green!60!black}{0.682}&\textcolor{green!60!black}{0.436}&0.320&0.624&\textcolor{green!60!black}{0.404}&0.318\\
Q5\_0 &0.564&0.377&0.274&0.502&0.342&0.281&0.678&0.433&0.323&0.620&\textcolor{green!60!black}{0.404}&0.316\\
Q5\_K &0.571&0.377&0.275&0.509&\textcolor{green!60!black}{0.343}&\textcolor{green!60!black}{0.296}&0.679&0.434&0.322&\textcolor{green!60!black}{0.625}&0.401&0.321\\
Q4\_0 &0.569&0.377&0.279&\textcolor{green!60!black}{0.511}&\textcolor{red}{0.338}&\textcolor{red}{0.268}&0.678&0.434&0.324&0.615&0.401&0.311\\
Q4\_K &\textcolor{green!60!black}{0.573}&0.380&\textcolor{green!60!black}{0.280}&0.502&\textcolor{green!60!black}{0.343}&0.294&0.680&0.433&\textcolor{red}{0.314}&0.619&\textcolor{green!60!black}{0.404}&0.317\\
Q3\_K &0.552&0.379&\textcolor{red}{0.264}&0.509&0.341&0.292&0.667&0.430&0.322&0.618&0.401&0.322\\
Q2\_K &\textcolor{red}{0.502}&\textcolor{red}{0.357}&0.272&\textcolor{red}{0.475}&0.339&0.278&\textcolor{red}{0.653}&\textcolor{red}{0.423}&\textcolor{green!60!black}{0.329}&\textcolor{red}{0.593}&\textcolor{red}{0.384}&\textcolor{red}{0.307}\\
\hline
\end{tabular}
\label{tab:ca}
\end{table*}

\paragraph{Correctness Agreement.} 
In addition to accuracy, Table~\ref{tab:ca} presents the results of correctness agreement that measures when both the base model and its quantized counterpart answer a given prompt correctly to capture the behavioral fidelity at the sample level. Mistral-7B performs the highest agreement scores across all benchmarks, highest on ARC at $0.682$ (Q6\_K) and on HellaSwag at $0.436$ (Q6\_K). Similarly, Llama-8B exhibits good behavioral alignment on ARC, ranging between $0.615$ (Q4\_0) and $0.625$ (Q5\_K). Llama-3B and Vicuna-7B maintain consistent agreement bands, reaching highest ARC agreements of $0.573$ (Q8\_0 and Q4\_K) and $0.511$ (Q6\_K and Q4\_0), respectively. Across all models, CA on Winogrande is clearly lower ($0.264$--$0.329$) compared to ARC and HellaSwag which reflect broader sample-level agreement differences on task-specific predictions even when accuracy remains stable. For low bit-width, Q3\_K introduces minor decreases in agreement across most datasets such as ARC dropping to $0.552$ on Llama-3B and $0.667$ on Mistral-7B—while having comparable behavior on HellaSwag. Q2\_K marks a larger drop in correctness agreement across nearly all configurations, with ARC suffering the largest declines to $0.502$ for Llama-3B, $0.475$ for Vicuna-7B, $0.653$ for Mistral-7B, and $0.593$ for Llama-8B, confirming substantial departure from base model predictions under extreme quantization. The table also demonstrate that quantization produce a noticeable behavioral divergence. For example, an agreement score of $0.682$, the highest observed across all evaluated models and benchmarks, implies a decision flip rate of $31.8\%$ relative to the base model. 
 
\begin{table*}[t]
\centering
\small
\setlength{\tabcolsep}{3.5pt}
\renewcommand{\arraystretch}{1.05}
\caption{Cohen's kappa of ARC, HellaSwag and Winogrande.}
\begin{tabular}{l ccc ccc ccc ccc}
\hline
& \multicolumn{3}{c}{Llama-3B} & \multicolumn{3}{c}{Vicuna-7B} & \multicolumn{3}{c}{Mistral-7B} & \multicolumn{3}{c}{Llama-8B} \\
Quant & Arc & Hell. & Wino. & Arc & Hell. & Wino. & Arc & Hell. & Wino. & Arc & Hell. & Wino. \\
\hline
Q8\_0 & \textcolor{green!60!black}{0.659} & \textcolor{green!60!black}{0.815} & 0.029 & \textcolor{green!60!black}{0.553} & 0.648 & 0.100 & 0.752 & 0.873 & 0.122 & 0.678 & 0.817 & 0.120 \\
Q6\_K & 0.658 & \textcolor{green!60!black}{0.815} & 0.029 & 0.547 & \textcolor{green!60!black}{0.659} & 0.106 & \textcolor{green!60!black}{0.760} & \textcolor{green!60!black}{0.877} & 0.124 & \textcolor{green!60!black}{0.687} & \textcolor{green!60!black}{0.821} & 0.103 \\
Q5\_0 & 0.640 & 0.808 & 0.028 & 0.548 & 0.651 & 0.068 & 0.751 & 0.875 & 0.132 & 0.675 & 0.817 & 0.115 \\
Q5\_K & \textcolor{green!60!black}{0.659} & 0.801 & 0.027 & 0.551 & 0.655 & 0.112 & 0.754 & 0.873 & 0.132 & \textcolor{green!60!black}{0.687} & 0.816 & 0.113 \\
Q4\_0 & 0.648 & 0.798 & 0.028 & 0.543 & \textcolor{red}{0.644} & \textcolor{red}{0.019} & 0.745 & 0.864 & 0.133 & 0.659 & 0.812 & 0.089 \\
Q4\_K & 0.658 & 0.805 & \textcolor{green!60!black}{0.038} & 0.534 & 0.652 & \textcolor{green!60!black}{0.114} & 0.754 & \textcolor{green!60!black}{0.877} & \textcolor{red}{0.111} & 0.668 & 0.817 & 0.107 \\
Q3\_K & 0.623 & 0.805 & \textcolor{red}{-0.016} & 0.543 & 0.656 & 0.103 & 0.736 & 0.867 & 0.146 & 0.656 & 0.812 & \textcolor{green!60!black}{0.125} \\
Q2\_K & \textcolor{red}{0.545} & \textcolor{red}{0.789} & 0.034 & \textcolor{red}{0.515} & 0.656 & 0.059 & \textcolor{red}{0.709} & \textcolor{red}{0.853} & \textcolor{green!60!black}{0.164} & \textcolor{red}{0.605} & \textcolor{red}{0.801} & \textcolor{red}{0.086} \\
\hline
\end{tabular}
\label{tab:kappa}
\end{table*}

Table~\ref{tab:kappa} reports Cohen's kappa measures the inter-rater agreement (or prediction consistency) between the base model and its quantized counterpart. 

Across all evaluated models, Cohen's kappa ($\kappa$) demonstrates moderate to substantial agreement on ARC ($\kappa \approx 0.534$--$0.760$) and HellaSwag ($\kappa \approx 0.644$--$0.877$) down to Q4\_K, with Mistral-7B reaching the highest behavioral consistency (peaking at $\kappa = 0.877$ on HellaSwag under Q6\_K and Q4\_K). On the other hand, Winogrande shows low agreement across all model families ($\kappa \le 0.164$, dropping as low as $\kappa = -0.016$ for Llama-3B at Q3\_K), indicating that prediction overlap on this dataset is very low regardless of quantization level. A gradual decline in $\kappa$ is observed across ARC and HellaSwag as precision decreases to Q3\_K, followed by a sharper drop at Q2\_K on ARC, where Llama-3B and Llama-8B decline to $\kappa = 0.545$ and $\kappa = 0.605$, underscoring severe degradation in decision boundary alignment under extreme 2-bit quantization.

The primary objective of quantization is to construct a quantized model that faithfully mirrors the parent model properties, "like father, like son", relationship where instance-level decisions remain identical. However, our empirical evaluation reveals a critical gap because while mid-range schemes (\texttt{Q8\_0} to \texttt{Q4\_K}) maintain near identical perplexity and accuracy, decision-level metrics expose clear divergence, peaking at only $0.682$ Correctness Agreement and $\kappa = 0.877$. This demonstrates that no quantized model acts identically to its base parent.

\section{Discussion}
This section synthesizes our experimental findings to assess the impact of quantization on language model behavior, statistical stability, and structural fidelity.

\textbf{Intersection of Structural and Behavioral Metrics.}
Weight divergence metrics directly account for the behavioral trends observed across perplexity and accuracy evaluations. The structural stability maintained down to \texttt{Q5\_K} explains the preserved decision-boundary fidelity ($\kappa \ge 0.80$ on HellaSwag) and near-baseline accuracy. As weight distributions begin to degrade at \texttt{Q4} and \texttt{Q3\_K} evidenced by rising KL divergence and kurtosis collapse sample-level alignment ($CA$ and $\kappa$) deteriorates before major drops in accuracy occur. Finally, extreme distribution distortions and severe drops in cosine similarity at \texttt{Q2\_K} coincide with global accuracy collapse and sharp perplexity spikes across all models. This confirm that functional degradation stems directly from structural weight matrix deformation.

\textbf{Sensitivity of Attention Projections.}
Our analyses reveal that the key ($\mathbf{K}$) and query ($\mathbf{Q}$) projections are the most sensitive to quantization. Under \texttt{Q3\_K} and \texttt{Q2\_K}, these layers exhibit the largest spikes in skewness, the greatest shifts in mean, and the steepest drops in kurtosis. In contrast, the value ($\mathbf{V}$) layer preserves its low-kurtosis profile even under moderate compression, while output ($\mathbf{O}$) projections remain relatively stable until the lowest quantization regimes. Weight divergence metrics reinforce this hierarchy, showing that cosine similarity and KL divergence deviate most severely for $\mathbf{K}$ and $\mathbf{Q}$.

\textbf{Silent Instability: Instance-Level Prediction Flipping.}
Quantization alters model decision-making relative to base checkpoints across all  bit level. A critical finding is the disconnect between performance metrics (perplexity and overall accuracy) and instance-level output consistency ($CA$ and $\kappa$). Aggregate accuracy frequently masks underlying prediction swaps. For instance, on Winogrande, Vicuna-7B achieves an identical accuracy of $0.52$ at both base and \texttt{Q6\_K}; however, instance-level analysis reveals a low Correctness Agreement ($CA = 0.294$) and Cohen's kappa ($\kappa = 0.106$). This demonstrates that identical accuracy can mask internal prediction flipping, where the quantized model loses correctness on one subset of prompts while gaining it on another by chance. 

\textbf{Impact of Quantization Schemes.}
Behavioral fidelity does not scale monotonically with nominal bit precision. For Vicuna-7B and Mistral-7B, highest agreement ($CA$) is not achieved at \texttt{Q8\_0}, but rather at intermediate levels such as \texttt{Q6\_K} ($CA = 0.511$ on ARC) and \texttt{Q5\_K} ($CA = 0.343$ on HellaSwag and Winogrande). Notably, modern K-quantization schemes (\texttt{Q3\_K} and \texttt{Q4\_K}) frequently match or outperform legacy schemes with higher precision (e.g., \texttt{Q5\_0}). This indicates that block-wise quantization architecture design plays as pivotal a role in preserving behavioral fidelity as total bit-width allocation.

\textbf{Limitations of Perplexity and Aggregate Accuracy Benchmarks.}
The majority of existing post-training quantization methods are evaluated using perplexity and accuracy. However, our findings demonstrate that these metrics are unable to capture behavioral consistency relative to base model. Consequently, evaluating quantized models on perplexity or accuracy overlooks internal prediction flips, masking substantial behavioral divergence and providing a false sense of reliability in downstream deployment.

\section{Conclusion}
In this paper, we formalized post-training quantization as an operator on parameter space and introduced \textit{Correctness Agreement}. We proved theoretically and empirically that matching top line accuracy masks substantial instance-level decision instability. Our evaluation reveals that while mid-range quantization schemes (\texttt{Q8\_0} to \texttt{Q4\_K}) preserve surface-level performance, decision-boundary alignment diverges under the hood. Distributional analysis links these behavioral shifts directly to layer-wise parameter perturbations, identifying query ($Q$) and key ($K$) attention projections as uniquely vulnerable components that undergo severe kurtosis collapse and distribution drift under aggressive compression (\texttt{Q3\_K} and \texttt{Q2\_K}). Ultimately, our findings demonstrate that no quantized model acts identically to its base parent. As LLMs are increasingly deployed in safety-critical and resource-constrained environments, evaluating instance-level behavioral fidelity must become a standard requirement alongside efficiency and aggregate metrics.


{\footnotesize
\bibliographystyle{unsrt}
\bibliography{llm}}



\end{document}